%% file: paper.tex
\documentclass[american,a4paper,runningheads,envcountsect]{llncs}

\usepackage{babel}
\usepackage[utf8]{inputenc}
\usepackage[T1]{fontenc}

\usepackage{paralist}
\usepackage{amsmath,amsthm,amssymb}
\usepackage{mathtools}
\usepackage{mathrsfs}
\usepackage{colonequals}
\usepackage{fca}
\usepackage{dsfont} 

\usepackage[colorlinks,citecolor=blue,urlcolor=black,hidelinks,linktocpage]{hyperref}
\usepackage[all]{hypcap}
\usepackage{cleveref}
\let\cref\Cref
\usepackage[babel,kerning=true]{microtype}
\usepackage[subtle]{savetrees}

\newtheorem{thm}{Theorem}[section]

\newtheorem{problm}{Problem}
\newtheorem{rem}[thm]{Remark}
\theoremstyle{definition}
\crefname{problem}{Problem}{Problems}

\theoremstyle{remark}

\renewcommand{\epsilon}{\varepsilon}
\renewcommand{\phi}{\varphi}
\newcommand*{\logeq}{\ratio\Leftrightarrow}

\newcommand{\abs}[1]{\lvert #1 \rvert}

\usepackage{csquotes}
\usepackage{booktabs}
\usepackage{todonotes}
\presetkeys{todonotes}{color=blue!5}{}

\usepackage{fancyvrb}

\usepackage[backend=bibtex,style=numeric-comp,doi=false,isbn=false,%
url=false]{biblatex}
\newcommand{\Scon}{\mathbb{S}}
\newcommand{\Ocon}{\mathbb{O}}

\DeclareMathOperator{\cdom}{co-dom}
\DeclareMathOperator{\Ext}{Ext}
\DeclareMathOperator{\Int}{Int}
\DeclareMathOperator{\Th}{Th}

\usepackage{tabularx, multirow,booktabs}
\begin{document}

\title{Ordinal Motifs in Lattices}

\date{\today}

\author{Johannes Hirth \inst{1,3}\orcidID{0000-0001-9034-0321}
  \and Viktoria Horn\inst{2,3}
  \and Gerd Stumme\inst{1,3}\orcidID{0000-0002-0570-7908}
  \and Tom Hanika\inst{1,3,4}\orcidID{0000-0002-4918-6374}
}
\institute{%
 Knowledge \& Data Engineering Group,
 University of Kassel, Germany\\[0.5ex]
 \and
 Gender/Diversity in Informatiksystemen,
 University of Kassel, Germany\\[0.5ex]
 \and
 Interdisciplinary Research Center for Information System Design\\
 University of Kassel, Germany\\
 \and
 Intelligent Information Systems, University of Hildesheim, Germany\\
 \email{$\{$hirth,stumme$\}$@cs.uni-kassel.de,
   viktoria.horn@uni-kassel.de, tom.hanika@uni-hildesheim.de}
}
\maketitle


\begin{abstract}
  Lattices are a commonly used structure for the representation and
  analysis of relational and ontological knowledge. In particular, the
  analysis of these requires a decomposition of a large and
  high-dimensional lattice into a set of understandably large
  parts. With the present work we propose /ordinal motifs/ as
  analytical units of meaning. We study these ordinal substructures
  (or standard scales) through order-embeddings and (full)
  scale-measures of formal contexts from the field of formal concept
  analysis. We show that the underlying decision problems are
  NP-complete and provide results on how one can incrementally
  identify ordinal motifs to save computational effort.  Accompanying
  our theoretical results, we demonstrate how ordinal motifs can be
  leveraged to achieve textual explanations based on principles from
  human computer interaction.
\end{abstract}


\section{Introduction}

The foundation of any formal analysis of data is the identification of
unique and meaningful substructures and properties. The realm of
ordinal structures, in particular lattices, is no exemption to
that. The field of Formal Conceptual Analysis (FCA), which derives
lattices from data tables, called formal contexts, is already very
well equipped with tools and notions for identifying and analyzing
important substructures. One essential tool of FCA is to provide a
user a lattice diagram of meaningful size, which can be interpreted
(or even explained). For obvious reasons, this approach defies any
applicability to data sets as they are commonly used today, as the
resulting lattices are comprised of thousands of elements. In
addition, the lattice diagram itself, as the primary means of
communication, presents a significant hurdle to interpretation for
untrained users. Common approaches tackle the first problem by data
set reductions within the data tables~\cite{pqcores,kumar2010concept}
or within the resulting lattice
structure~\cite{scale-explore,belohlavek2011selecting,
  DBLP:conf/ijcai/BelohlavekT13,kuznetsov2018interestingness}. The
second problem is treated, to some extent, by improving order diagrams
of lattices through locally optimal layouts~\cite{lattice-layouts} or
by (interactively) collapsing~\cite{FCA-visual}. Theses approaches
are most often motivated from graph-theoretical points of view or
apply statistical methods. In general, they do not explicitly account
for identifying and employing basic ordinal sub-structures within the
lattice, such as nominal scales, ordinal scales, or inter-ordinal
scales. Even more, they do not allow the analysis of a lattice using
arbitrary ordinal patterns.

With our work, we provide the theoretical foundations for analyzing
(concept) lattices by means of ordinal substructures. We call this
approach, in analogy to the notion established in network
science~\cite{holland1974statistical,holland1976local,milo2002network},
\emph{ordinal motifs}. However, in contrast to network science, where
motifs are recurrent and statistically significant subgraphs (or
patterns), we understand motifs as user-defined set $\mathcal{O}$ of
ordered sets, usually represented as formal
contexts~\cite{fca-book}. The elements of this set can be of different
sizes and (ordinal) complexities. They shall allow to analyze any
lattice, or ordinal structure, by means of frequency and sizes of
ordinal patterns. Thus, the set $\mathcal{O}$ can be considered as an
\emph{ordinal tool-set}.  In addition to the standard scales mentioned
above, any pattern deemed relevant by a user lends itself to
$\mathcal{O}$. However, we show in our work that already for standard
scales the recognition of these motifs is a difficult problem.

In order to represent and compute ordinal motifs we employ recently
developed methods from the realm of conceptual measurement, i.e.,
\emph{scale-measures}~\cite{smeasure}. These are continuous maps
between closure systems and can be used to map an ordinal structure,
or parts of it, to an ordinal motif. As these maps are continuous they
ensure that the relation between objects and attributes in a motif is
correct with respect to the underlying conceptual structure of the
original data set.

In terms of theoretical results, we have shown the computational
complexity of several decision problems for recognizing and finding
scale-measures. In particular, we show that for finding a
scale-measures for a given ordinal motif we have to solve an
NP-complete problem. We show that motifs which have the special
property of belonging to a heriditary class of scales offer many
advantages in computation. 

Finally, to demonstrate the applicability of the ordinal motif method
we demonstrate how to find them and provide basic interpretations
motifs based on standard scales in a medium sized data set, the spice
planner data set~\cite{pqcores}.

Overall, our work proposes a new approach to the analysis of (large)
lattices and, in particular, ordinal structures, in order to improve
their human interpretability.

\section{Ordinal Motifs}
In the beginning of this section we recall all necessary basics from
Formal Concept Analysis (FCA). After finishing this paragraph, readers
who are familiar with FCA may skip directly to~\cref{sec:motifs}, in
which we introduce our notion for \emph{ordinal motifs}. To work with
these, we draw from the notion of scale-measures, i.e., continuous
maps between closure systems. An extension of these mappings with a
local version allows us to prove the computational complexities for
recognizing ordinal motifs.

\subsection{Formal Concept Analysis}
\label{sec:fca}
Throughout this paper we use the notation as introduced
by~\textcite{fca-book}.  That is, in the following
$\context\coloneqq(G,M,I)$ denote a \emph{formal context}. The sets
$G$ and $M$ are respectively called object and attribute set, and the
binary relation $I\subseteq G\times M$, called \emph{incidence},
indicates if an object $g\in G$ \emph{has} an attribute $m\in M$ by
$(g,m)\in I$. The incidence relation $I$ gives rise to two important
maps, called derivation operators,
$\cdot':\mathcal{P}(G)\to\mathcal{P}(M)$, $A\mapsto A'\coloneqq\{m\in
M\mid \forall g\in A:(g,m)\in I)\}$, and (the dual)
$\cdot':\mathcal{P}(M)\to \mathcal{P}(G)$, $B\mapsto B'\coloneqq\{g\in
G\mid\forall m\in B:(g,m)\in I\}$. There are situations where multiple
formal contexts are used, in these cases will explicitly note which
incidence relation is applied, i.e., we write $A^{I}$ instead of $A'$.

The namesake for FCA are the \emph{formal concepts}, i.e., pairs
$(A,B)\in\mathcal{P}(G)\times\mathcal{P}(M)$ where $A'=B$ and
$B'=A$. The sets $A,B$ are called \emph{extent} and \emph{intent}
respectively. The set of all concepts of a formal context $\context$
is denoted by $\BV(\context)$, which is a lattice ordered set, called
\emph{concept lattice}, given the following relation: $(A,B)\leq
(C,D)\logeq A\subseteq C$. We denote by $\Ext(\context)$ the set of
all extents and by $\Int(\context)$ the set of all intents. Both sets
each form an closure system and there is an isomorphism between
them. The corresponding closure operators are the respective
compositions of the derivations.

For a closure system $\mathcal{C}$ on $G$ we call $\mathcal{D}$ a
\emph{finer} closure system (denoted by $\mathcal{C}\leq\mathcal{D}$)
iff $\mathcal{D}$ is a closure system on $G$ and
$\mathcal{A}\subseteq\mathcal{D}$. Conversely we say $\mathcal{C}$ is
\emph{coarser} than $\mathcal{D}$. In the particular case where
$\mathcal{C}$ is a closure system on $H\subseteq G$ and
$\mathcal{C}=\{H\cap D\mid D\in \mathcal{D}\}$ we call $\mathcal{C}$ a
\emph{sub-closure system} of $\mathcal{D}$.

Our work uses in particular \emph{scale-measures} (\cref{def:sm}), for
which we needs maps between different object sets of different formal
contexts. For such a map $\sigma: G_1\to G_{2}$ we remind the reader
that the image of a set $A\subseteq G_{1}$ is $\sigma(A) \coloneqq
\bigcup_{g\in A}\sigma(g)$. Moreover, for any
$\mathcal{A}\subseteq\mathcal{P}(G)$ we set
$\sigma(\mathcal{A})\coloneqq \{\sigma(A)\mid
A\in\mathcal{A}\}$. Essential for scale-measures will be the pre-image
of sets $A\in G_{2}$, i.e.,
$\sigma^{-1}(A)\coloneqq\{\sigma^{-1}(g)\mid g\in A\}$.

\subsection{Mapping and Representation}
\label{sec:motifs}
The overall goal for ordinal motifs is to identify frequent recurring
ordinal patterns that allow for analyzing large and complex ordinal
structures. Thus, ordinal motifs are themselves ordered structures.

There are various ways for representing ordinal structures. To draw
from the powerful theoretical and algorithmic tool-set of FCA, we
consider any ordered set $(P,\leq)$ represented as context, i.e.,
$(P,P,\leq)$. This context is called the \emph{general ordinal scale}
and its concept lattice $\BV(P,P,\leq)$ is the smallest complete
lattice in which $(P,\leq)$ can be order-embedded~\cite[Theorem
4]{fca-book}.

\begin{definition}[full scale-measure (Definition 91 \cite{fca-book})]
  \label{def:sm}
  For two formal contexts $\context,\context[S]$ a map
  $\sigma: G_{\context}\to G_{\Scon}$ is a scale-measure iff for all
  $A\in \Ext(\Scon)$ the pre-image $\sigma^{-1}(A)$ is in
  $\Ext(\context)$. A scale-measure is full iff
  $\Ext(\context) = \sigma^{-1}(\Ext(\Scon))$.
\end{definition}

The formal context $\context[S]$ in the definition above is called
\emph{scale}, hence the name scale-measure. However, there is no
restriction on what can be used as a scale context. Given this tool of
continuous maps we want to express ordinal motifs in the language of
formal contexts. In doing so, we want to consider the following
aspects: \emph{scope} and \emph{coverage}. We will first give an
informal explanation of the two properties and then derive the
mathematical tools and a precise problem definition.

Starting from a given ordinal data set $\context[D]\coloneqq
(G_{\context[D]},M_{\context[D]},I_{\context[D]})$ and an ordinal
motif $\context[S]$, both in the form of a formal context, the scope
of the ordinal motif is
\begin{itemize}
\item \textbf{global}, if it covers the entire data, i.e., all objects
  $G_{\context[D]}$, or
\item \textbf{local}, if it covers only parts of $G_{\context}$. 
\end{itemize}
Since it is very difficult to find a motif that captures the complex
structure of a given data set, one usually relies on local
motifs. However, scale-measures are incapable of capturing an ordinal
motif only locally, i.e., only on a part of the data. We will
therefore introduce scale-measures based on partial maps $\sigma:
H\subseteq G_{1}\to G_{2}$ in a few moments.
The coverage of an ordinal motif concerns the portion of the ordinal
structure that is captured by the motif. We say an ordinal motif
\begin{itemize}
\item has \textbf{full coverage}, if every element of the ordinal structure
  of $\context[D]$, i.e., of the concept lattice, has a correspondence
  in the ordinal structure of the motif, or
\item has \textbf{partial coverage}, otherwise.
\end{itemize}
For example, the latter case exists if there are concepts of
$\context[D]$ that are not the pre-image of an extent of
$\context[S]$.
In case there is a full scale-measure from a context $\context[D]$ to
a context $\context[S]$, we can infer that the closure system of
$\context[D]$ on $G_{\context[D]}$ is, except for relabeling,
identical to that of $\context[S]$. A scale-measure from $\context[D]$
to $\context[S]$, on the other hand, only guarantees that the closure
system of $\context[D]$ on $G_{\context[D]}$ has at least all closed
sets that the context $\Scon$ has, up to relabeling.

\begin{rem}[Surjective Scale-Measures]
  It is reasonable to consider only surjective maps when using scale
  contexts for ordinal motifs. Since objects that are not contained in
  the image of the scale-measure $\sigma$ do not contribute to the set
  of reflected extents, dropping surjectivity would allow for trivial
  maps into ordinal motifs.
\end{rem}

In order to introduce a local variant of scale-measures, we need to
fix some notation. Given a formal context $\context=(G,M,I)$, by
$\context{[H,N]}$ we refer to the \emph{induced subcontext} of
$\context$ on $H\subseteq G$ and $N\subseteq M$, i.e., $(H,N,I\cap
H\times N)$.

\begin{definition}[local scale-measures]
  For $\context=(G_{\context},M_{\context},I_{\context})$ and scale
  context $\Scon$ a map $\sigma:H\to G_{\Scon}$ is a \emph{local
    scale-measure}, if
  \begin{enumerate}
  \item $H\subseteq G_{\context}$ and
  \item $\sigma$ is a scale-measure from $\context{[H,M_{\context}]}$ to $\Scon$.
  \end{enumerate}
  We say a local scale-measure is \emph{full}, iff $\sigma$ is a full
  scale-measure from $\context{[H,M_{\context}]}$ to $\Scon$.
\end{definition}

For local and full scale-measures the relation between the respective
concept lattices is captured by the following proposition. In it the
relation symbol $\cong$ is used to indicate that two ordered sets are
isomorphic and
$(\mathcal{A},\subseteq)\leq_{\Ext}(\mathcal{B},\subseteq)$ denotes
and $(\mathcal{A},\subseteq)$ is a sub-closure system of
$(\mathcal{B},\subseteq)$. 

\begin{proposition}[local and full scale-measure]
  \label{lem:local-full-sm-equivalences}
  For contexts $\context$, $\Scon$, a surjective full scale-measure
  $\sigma$ from $\context$ to $\Scon$, and the closure operator
  $\phi_{\context}$ on $\Ext(\context)$ we find that
  \begin{equation}
    \label{eq:sm1}
    (\Ext(\context),\subseteq)\cong (\sigma(\Ext(\context)),\subseteq)
    \cong (\Ext(\Scon),\subseteq) \cong
    (\sigma^{-1}(\Ext(\Scon)),\subseteq).
  \end{equation}
  For a local surjective scale-measure $\sigma$ with $H\subseteq
  G_{\context}$ we find that
  \begin{equation}
    \label{eq:sm2}
    (\Ext(\Scon),\subseteq)\cong(\sigma^{-1}(\Ext(\Scon)),\subseteq)\leq_{\Ext}
    (\Ext(\context\lbrack H,M_{\context}\rbrack),\subseteq)
  \end{equation}
  and that
    \begin{equation}
    \label{eq:sm3}
    (\Ext(\context\lbrack H,M_{\context}\rbrack),\subseteq)\cong
    (\{\phi_{\context}(E)| E\in \Ext(\context\lbrack H,M_{\context}\rbrack )\},\subseteq).
  \end{equation}

\end{proposition}
\begin{proof}
  \begin{compactitem}
  \item[\cref{eq:sm1}:] From the surjectivity of $\sigma$ we can
    deduce via~\cite[Proposition 118]{fca-book} that the map
    $\sigma^{-1}$ exists, which is injective. This also means that
    every extent $E\in\Ext(\context[S])$ is mapped to a unique extent
    $\hat E\in \Ext(\context)$. Moreover, since $\sigma$ is a full
    scale-measure, every extent of $\context$ is a also a pre-image of
    an extent of $\context[S]$. From this it follows that $\sigma^{-1}$
    bijectively maps the extents from $\context[S]$ and $\context$.

    Finally, since there is no $g\in G_{\context[S]}$ with
    $\sigma^{-1}(g)=\emptyset$, for every $E,\hat E\in
    \Ext(\context[S])$ with $E\subseteq \hat E$ it is true that
    $\sigma^{-1}(E)\subseteq \sigma^{-1}(\hat E)$.
  \item[\cref{eq:sm2}:] From left to right, the first $\cong$-relation
    can be inferred from~\cref{eq:sm1}. The second, i.e.,
    $\leq_{\Ext}$, follows by definition of scale measures.
  \item[\cref{eq:sm3}:] For the final $\cong$-relation we can note
    that for $A\in\Ext(\context{[H,M_{\context}]})$ the difference
    $\phi_{\context}(A)\setminus A$ is in $G\cap H$. This means, the
    closure of $A$ in $\Ext(\context)$ adds only elements from
    $G\setminus H$. Thus, since $\phi_{\context}$ is a closure
    operator we find that for $A,C\in \Ext(\context{[H,M_{\context}]})$ with
    $A\subset C$ we have $\phi_{\context}(A) \subset
    \phi_{\context}(C)$. Hence, $\phi_{\context}:
    \Ext(\context{[H,M_{\context}]}) \to \Ext(\context)$ is an injective map and
    by restricting the codomain we find a bijective map
    $\hat\phi_{\context}: \Ext(\context{[H,M_{\context}]}) \to
    \{\phi_{\context}(E)\mid E\in\Ext(\context{[H,M_{\context}]})\}$.
  \end{compactitem}
\end{proof}

\cref{lem:local-full-sm-equivalences} reveals the relations between a
context $\context$ and an ordinal motif $\Scon$. It summarizes known
results and shows new equivalences. For the case of full
scale-measures we now know that the closure systems of $\context$ and
$\Scon$ are equal up to relabeling. Hence, to analyze an ordinal
structure $\context[D]$ via ordinal motifs in the full scale-measure
setting would mean to simply speak about $\context[D]$ with different
labels. For local case we find that scale-measures
reflect a coarser closure system.

The following problem summarizes the technical observations so far and
(finally) states all notions for ordinal motif.

\begin{problm}[Finding Ordinal Motifs]\label{problem:explain}
  Given a formal context $\context$ and an ordinal motif $\context[S]$
  find a surjective map from $\context$ into $\context[S]$
  that is a:

  \begin{center}
    \begin{tabular}{>{\raggedright\hspace{0pt}\itshape}p{2cm}|>{\raggedright\hspace{10pt}}p{5cm}l}
      &$\mathrm{global}$&$\mathrm{local}$\\ \midrule
      $\mathrm{partial}$& scale-measure&  local scale-measure\\
      $\mathrm{full}$   &  full scale-measure & local full scale-measure\\
    \end{tabular}
  \end{center}
\end{problm}


In the next section, we employ those maps to substitute
elements in ordinal motif explanations by the real world objects. The
result is then an explanation of the data set.

\subsection{Recognizing Scale-Measures}\label{sec:verify}
Recognizing scale-measures is the problem for deciding if for a given
formal context $\context$ a scale $\context[S]$ and a map $\sigma$ is
a scale-measure of $\context$. This problem has been studied
in~\textcite{smeasure} and the time complexity was found to be in
$O(\abs{\context}\cdot\abs{\Scon})$. On top of that one has to check
for full scale-measures if for each meet-irreducible extent $A$ of
$\context$ that $\sigma(A)\in \Ext(\Scon)$ and
$\sigma^{-1}(\sigma(A))=A$. However, this problem is dual to the
original scale-measure recognition problem. Thus verifying full
scale-measure can be done in time
$O(\abs{\context}\cdot\abs{\Scon})$. The check for local (full)
scale-measures has the same cost, since it is the same check but for
subcontext $\context{[H,M_{\context}]}\leq \context$.

\begin{corollary}[Recognizing Ordinal Motifs]
  Given two formal contexts $\context,\Scon$ and a map $\sigma:
  G_{\context}\to G_{\Scon}$, deciding if $\sigma$ is (local) (full)
  scale-measure is in $O(\abs{\context}\cdot\abs{\Scon})$.
\end{corollary}

\subsubsection*{Scale-Measures and Implicational Theories}
Before we now turn to finding ordinal motifs in ordinal data, i.e.,
finding scale-measures, we want to point out one more practical
relevant observation with the proposition at the end of this
subsection. In practice, context like data sets are large, however,
mostly only in one dimension. The usual case is that the number of
objects in a formal context is many times larger than the number of
attributes. The reverse case, of course, also occurs. The most
expensive computation for context and scales is the derivation, in
particular in the direction of the larger dimension, i.e., objects or
attributes. We therefore want to present an alternative representation
using implications in contexts.

In a formal context $\context=(G,M,I)$, we say a pair
$(A,B)\in\mathcal{P}(M)\times\mathcal{P}(M)$ is a \emph{valid
  attribute implication}, usually denoted by $A\to B$, iff
$A'\subseteq B'$. In other words, all objects having the attribute set
$A$ do also have the attributes $B$. The set of all valid attribute
implications is commonly denoted by $\Th(\context)$. Of course, on may
analogously define and use \emph{object implications}, as we will do
in the following. Hence, $\Th(\context)$ refers to the set of valid
object implications in $\context$.

To syntactically link implications with scale measures, we use the
short hand $\sigma^{-1}(A\to B)\coloneqq \sigma^{-1}(A)\to
\sigma^{-1}(B)$. For the theory $\Th(\context)$ we define
$\sigma^{-1}(\Th(\context))\coloneqq \{\sigma^{-1}(A\to B)\mid A\to B
\in \Th(\context)\}$.

\begin{proposition}[Recognizing (full) Scale-Measures using Implications]\label{lem:verify-sm-implications}
  For a context $\context$ a scale $\Scon$ and a map $\sigma:
  G_{\context}\to G_{\Scon}$ we find that

  \begin{enumerate}[i)]
    \item $\sigma \textit{ is a scale-measure} \iff \sigma^{-1}(\Th(\Scon)) \vdash \Th(\context)$
    \item $\sigma \textit{ is a full scale-measure} \iff \Th(\context)\equiv \sigma^{-1}(\Th(\Scon))$.
  \end{enumerate}
\end{proposition}
\begin{proof}
  First, we note that for two implicational theories $\Th_1,\Th_2$,
  i.e., transitive closures of implication sets, is holds that
  $\Th_1\subseteq \Th_2\iff \Th_2\vdash \Th_1$. Secondly, we note that
  there is a Galois connection between the lattice of all
  implicational theories and the lattice of all closure
  systems~\cite[Theorem 57]{lattice_of_closure_systems} to which the
  hierarchy of scale-measures is isomorph \cite[Proposition
  11]{smeasure}.

  \begin{enumerate}[i)]
  \item The map $\sigma$ is a scale-measure iff the closure system
    $\sigma^{-1}(\Ext(\Scon))$ is a sub-closure system of
    $\Ext(\context)$ on $G_{\context}$. Given our preliminary
    considerations this is the case if and only if the theory of
    $\Th(\context)$ is entailed in $\sigma^{-1}(\Ext(\Scon))$,
    i.e., $\sigma^{-1}(\Th(\Scon)) \vdash \Th(\context)$.
  \item The map $\sigma$ is a full scale-measure iff the closure
    system $\sigma^{-1}(\Ext(\Scon))$ is equal to $\Ext(\context)$
    (cf. \cref{lem:local-full-sm-equivalences}). Given our preliminary
    considerations this is the case if and only if their theories are
    equal.
  \end{enumerate}
\end{proof}

With the help of \cref{lem:verify-sm-implications} one may use already
existent logical inference checkers for the verification of (local)
(full) scale-measures.

\subsection{Ordinal Motif Problems}
Starting from~\cref{problem:explain}, we now want to formulate a
decision problem to investigate the complexity of
~\cref{problem:explain}. In the following we refer by \emph{DSM} to
the decision problem, if for two formal contexts $\context$ and
$\context[S]$ there exists an surjective scale-measure from $\context$
to $\context[S]$, the \emph{Deciding Surjective Scale-Measures}
problem. Analogously, we refer by \emph{DfSM} to the decision problem,
if for two formal contexts there exists a full surjective
scale-measure.  As remarked before, considering scale-measures that
are not surjective is not meaningful for ordinal motifs. In particular
for the problem definition, there is a always trivial scale-measure
that maps all objects onto a single object of $\Scon$.

\begin{theorem}[Ordinal Motif Problems]
  For two formal contexts $\context$ and $\Scon$, DSM and DfSM are
  NP-complete.
\end{theorem}
\begin{proof}
  To avoid any peculiarities, we consider in the following reductions
  graphs of size at least three.
  \begin{enumerate}[a)]

  \item \textbf{hardness:} To show NP-hardness of the DSM problem, we
    reduce the subgraph isomorphism (SI) problem to DSM. For two
    Graphs $G,H$ consider the formal context
    $\mathbb{G}=(V_{G}\cup\{\bot\},E_{G}\cup \{\{v\}\mid v\in
    V_{G}\}\cup \{\emptyset\},\in)$ and analogously constructed formal
    context $\mathbb{H}$. The set of extents of $\mathbb{G}$ is equal
    to $\{\{v\}\mid v\in V_{G}\}\cup E_G \cup
    \{\emptyset,V_{G}\cup\{\bot\}\}$.  This reduction is polynomial in
    the size of $G,H$.
    \begin{compactenum}
    \item[$\Rightarrow$] Let $\sigma$ be a surjective scale-measure of
      $\mathbb{G}$ into $\mathbb{H}$. Then
      $\sigma^{-1}(\Ext(\mathbb{H}))\subseteq
      \sigma^{-1}(\Ext(\mathbb{G}))$. In particular for every
      $e\in E_H$ we have $\sigma^{-1}(e)\in \Ext(\mathbb{G})$. Since
      $\sigma$ is surjective, we can infer that
      $2\leq \abs{\sigma^{-1}(e)} < \abs{V_G}$. The only extents with
      a cardinality in that interval are the edge extents of
      $\mathbb{G}$. Thus $\sigma^{-1}(e) \in E_G$ and all nodes of $e$
      have a unique pre-image. Since $E_H\subseteq \Ext(\mathbb{H})$,
      all nodes with at least one edge have a unique pre-image. WLOG
      we assume that the pre-image of all $v\in V_{H}$ have a unique
      pre-image, otherwise change the map $\sigma$ for all but one
      node to $\bot$. 
      Hence the map $\sigma^{-1}: V_H\to V_G$ is edge preserving and
      an isomorphism of $(H,E_H)$ into a subgraph of $G$, i.e., into the subgraph given by
      $(\cdom(\sigma^{-1}),\{e\in E_G\mid \exists l\in E_H:\sigma^{-1}(l)=e\})$.
    \item[$\Leftarrow$] Let $\sigma$ be an isomorphism of
      $\mathbb{H}$ into a subgraph of $\mathbb{G}$, i.e., an edge
      preserving map from $H$ into $G$. Based on this consider the map
      $\theta:V_G\cup \{\bot\}\to V_H\cup \{\bot\}$ where
      $\theta(v)=\sigma^{-1}(v)$ and $\bot$ otherwise. The map
      $\theta$ is surjective by definition. For the node extents, the
      empty extent and the top extent $V_H\cup\{\bot\}$ of
      $\mathbb{H}$ we have that their pre-images are in
      $\Ext(\mathbb{G})$. For an extent $e$ in $E_{H}$ we have that
      $\theta^{-1}(e)=\sigma(e)\in E_G$, since $\sigma$ is edge
      preserving. Thus $\theta$ is a surjective scale-measure from
      $\mathbb{G}$ into $\mathbb{H}$.
    \end{compactenum}
    \textbf{completeness:} An algorithm for identifying if there is a
    surjective scale-measure for two context $\Ocon,\context$ can be
    constructed by guessing non-deterministically a mapping
    $\sigma$. The check for a surjective scale-measure can be done
    deterministically in polynomial time in the size of both contexts.
  \item \textbf{hardness:} To show NP-hardness of the DfSM problem, we
    reduce the induced subgraph isomorphism (ISI) problem to the DSM
    problem. For two Graphs $G,H$ consider the contexts
    $\mathbb{G}=(V_{G},E_{G}\cup \{\{v\}\mid v\in V_{G}\}\cup
    \{\emptyset\},\in)$ and $\mathbb{H}$ analogously. The set of
    extents of $\mathbb{G}$ is equal to $\{\{v\}\mid v\in V_{G}\}\cup
    E_G \cup \{\emptyset,V_{G}\}$.  This reduction is polynomial in
    the size of $G,H$.
    \begin{compactenum}
    \item[$\Rightarrow$] Let $\sigma$ be a full scale-measure of
      $\mathbb{H}$ into $\mathbb{G}$. Then for every $v\in V_H$ the
      extent extent $\{v\}\in \Ext(\mathbb{H})$ is the pre-image of an
      extent $A$ of $\Ext(\mathbb{G})$. Since $v\in \sigma^{-1}(A)$ we
      have $\sigma(v)\in A$ and from $\{v\}= \sigma^{-1}(A)$ we can
      infer that there exists no other $w\in V_H$ with $w\neq v$ and
      $\sigma(w)=\sigma(v)$. Thus $\sigma$ is injective. 

      For an edge $e\in E_G$ where $e\subseteq \cdom(\sigma)$ we have
      $\sigma^{-1}(e)\in \Ext(\mathbb{H})$ and since $\sigma$ is
      injective we can infer $\abs{\sigma^{-1}(e)}=2$ and thus
      $\sigma^{-1}(e)\in E_H$. For an edge $e\in E_H$ there must be an
      $A\in \Ext(\mathbb{G})$ with $\sigma^{-1}(A)=e$. Thus
      $\sigma(e)\subseteq A$. Since the only extents of $\mathbb{G}$
      for which this applies are $V_G$ extents of cardinality
      two
      , i.e., the edges of
      $G$. Thus $\sigma(e)\in \Ext(\mathbb{G})$ and further
      $\sigma(e)\in E_G$. Concluding, $\sigma$ is an isomorphism
      between $H$ and $\sigma(H)$.
    \item[$\Leftarrow$] Let $\sigma$ be an isomorphism between $H$
      and an induced subgraph of $G$. Then for every $v\in V_G$ is
      $\sigma^{-1}(\{v\})$ either in $V_H$ or empty since $\sigma$ is
      injective. For edges $e\in E_G$ where $e\subseteq \cdom(\sigma)$
      we have that $\sigma^{-1}(e)\in E_H\subseteq \Ext(\mathbb{H})$
      since $\sigma$ is an isomorphism restricted to
      $\cdom(\sigma)$. In case $e\subseteq \cdom(\sigma)$ does not
      hold, the pre-image is equal to a node or the emptyset. Thus
      $\sigma^{-1}(E_G)\subseteq\Ext(\mathbb{H})$. Furthermore,
      $\sigma^{-1}(\emptyset)=\emptyset \in \Ext(\mathbb{H})$ and
      $\sigma^{-1}(V_G)=V_H \in \Ext(\mathbb{H})$. Thus $\sigma$ is a
      scale-measure of $\mathbb{H}$ into $\mathbb{G}$. For an edge
      $e\in E_H$ we have that
      $\sigma^{-1}(\{\sigma(v)\mid v\in e\})=e$ and
      $\{\sigma(v)\mid v\in e\}\in E_G\subseteq \Ext(\mathbb{G})$
      since $\sigma$ is isomorphism restricted to $\cdom$. Thus
      $\sigma$ is a full scale-measure.
    \end{compactenum}
    \textbf{completeness:} An algorithm for identifying if there is a
    full scale-measure for two context $\Ocon,\context$ can be
    constructed by guessing non-deterministically a mapping
    $\sigma$. The check for a full scale-measure can be done
    deterministically in polynomial time in the size of both contexts.
  \end{enumerate}
\end{proof}

Unfortunately all these problems are NP-complete which makes the task
computational costly. Further studying the computational complexity of
the local variations or preserving maps will not help either, since
these problems are only slight variations from the ones studied
here. For example a reduction for local scale-measures can be done
analogously by removing the $\bot$ node that was itnroduced to capture
all nodes that were not in the co-domain of the isomorphism, or
consider a map $\sigma$ from $\mathbb{G}$ to $\mathbb{H}$ of the local
variant of the full scale-measures reduction.

Now that we understand the computational complexities for both
problems, we want to present an interesting property of scales that
may actually help to reduce the computational efforts.

\section{Heredity of Ordinal Motifs}
For the use of ordinal motifs for the analysis of a data set, it is
meaningful to consider a set of ordinal motifs
$\mathcal{O}$. Moreover, there are particularly meaningful families of
scale contexts, such as the standardized (or elementary) scales of
ordinal type~\cite{fca-book}. These have a special property, called
\emph{heredity}, i.e., every subscale of a scale belonging to the same
family is equivalent to a scale of the same family.

In this section we will demonstrate how the notion for heredity of
scales impacts scale-measures.

\begin{lemma}[Heredity of Scale-Measures]\label{lem:heredity}
  Let $\context$ be a formal context and $\Scon$ a scale with
  $\sigma: G_{\context}\to G_{\Scon}$ a surjective (full)
  scale-measure. For any $H\subseteq G_{\context}$ is the map
  $\sigma\mid_{H}$ a surjective (full) scale-measure from
  $\context{[H,M_{\context}]}$ into $\Scon{[\sigma(H),M_{\Scon}]}$.
\end{lemma}
\begin{proof}
  First we show that $\sigma\mid_{H}$ is a scale-measure from
  $\context{[H,M_{\context}]}$ into
  $\Scon{[\sigma(H),M_{\Scon}]}$. Since $\Scon{[\sigma(H),M_{\Scon}]}$
  is an induced subcontext of $\Scon$ with equal attribute set, we can write every extent
  $A\in \Ext(\Scon{[\sigma(H),M_{\Scon}]})$ as the intersection
  ${\check A} \cap \sigma(H)$ for some ${\check A}\in \Ext(\Scon)$. The
  pre-image $(\sigma\mid_{H})^{-1}({\check A} \cap \sigma(H))$ is equal
  to
  $(\sigma\mid_{H})^{-1}({\check A})\cap
  (\sigma\mid_{H})^{-1}(\sigma(H))$. Since ${\check A}$ and $\sigma(H)$
  are entailed in the image of $\sigma$ on $H$ we can follow that
  $(\sigma\mid_{H})^{-1}({\check A}) = \sigma^{-1}({\check A})$ and
  $(\sigma\mid_{H})^{-1}(\sigma(H)) = H$. Moreover, since $\sigma$ is
  a scale-measure we can follow that $\sigma^{-1}({\check A})$ is an
  extent of $\context$. Summarizing, the preimage
  $(\sigma\mid_{H})^{-1}(A)$ is equal to the intersection of an extent
  of $\context$ and $H$. Hence, $(\sigma\mid_{H})^{-1}(A)$ is an
  extent of $\context{[H,M_{\context}]}$ and $\sigma\mid_{H}$ a
  scale-measure of $\context{[H,M_{\context}]}$ into
  $\Scon{[\sigma(H),M_{\Scon}]}$.

  In case $\sigma$ is a full scale-measure it remains to be shown that
  for every $D\in \Ext(\context{[H,M]})$ there exists a
  $C\in \Scon{[\sigma(H),M_{\Scon}]}$ with
  $(\sigma\mid_{H})^{-1}(C)=D$. We can write the extent $D$ as the
  intersection $\check D \cap H$ where $\check D\in
  \Ext(\context)$. Since $\sigma$ is a full scale-measure we can
  follow for $\check D$ that there is a $\check C\in \Ext(\Scon)$ with
  $\sigma^{-1}(\check C)=\check D$. Since $\Scon{[\sigma(H),M_{\Scon}]}$
  is an induced subcontext of $\Scon$ with equal attribute set we find
  that ${\check C} \cap \sigma(H)$ is an extent of
  $\Scon[\sigma(H),M_{\Scon}]$. Thus, for
  $C\coloneqq {\check C} \cap \sigma(H)$ we find that
  $(\sigma\mid_{H})^{-1}(\check C\cap
  \sigma(H))=(\sigma\mid_{H})^{-1}(\check C)\cap
  (\sigma\mid_{H})^{-1}(\sigma(H))$ and furthermore that
  $(\sigma\mid_{H})(\check C)\cap
  (\sigma\mid_{H})^{-1}(\sigma(H))=\check D \cap H = D$.  Hence,
  $\sigma\mid_{H}$ is a full scale-measure.

  The map $\sigma\mid_{H}$ is surjective, since the object set of
  $\Scon{[\sigma(H),M_{\Scon}]}$ is equal to the co-domain of $\sigma\mid_{H}$.
\end{proof}

\begin{proposition}[Heredity of Surjective Scale-Measures]\label{prop:incremental-sur-sm}
  Let $\context$ be a formal context and $\Scon$ a heredity scale with
  $\sigma: G_{\context}\to G_{\Scon}$ a surjective (full)
  scale-measure. Then for any $H\subseteq G_{\context}$ is the map
  $\sigma\mid_{H}$ a surjective (full) scale-measure from
  $\context{[H,M_{\context}]}$ into an ordinal motif of the same
  family as $\Scon$.
\end{proposition}
\begin{proof}
  This proposition follows directly from \cref{lem:heredity} and the
  definition of heredity scales.
\end{proof}

This proposition is essential when applying ordinal motifs for the
analysis of ordinal data set using heredity scales. When
computing all candidates for (full) scale-measures this statement
allows to discard a large proportion. Many families of scales, such as
the nominal scales, ordinal scales, inter-ordinal scales,
contra-nominal scales, etc, have the heredity property~\cite[Proposition 123]{fca-book}. Unfortunately,
the crown scales do not have this property. The problem for deciding
surjective scale-measures for  crown scales is related to the
hamiltonian path problem and therefore we do not expect there to be an
easy solution to this problem. However, as far as our preliminary
investigations on real-world data suggest, large crown scales are
rare. Nonetheless, this claim has to be studied more thoroughly.

\section{Applying Ordinal Motifs to Data Sets}
\label{sec:explain}
We demonstrate the applicability of ordinal motifs on real-world data
using a medium sized formal context: the \emph{spices planner} data
set~\cite{pqcores}. This context contains 56 meals (objects) and 37
spices and food categories (attributes). The incidence encodes that a
spice is recommended to be used to cook a meal or a meal belongs to a
food category. The context has 531 formal concepts. We conduct our
experiment on the dual context, i.e., $\context^{d}\coloneqq
(M,G,I^{d})$, to derive ordinal motifs within the spices and food
categories. For our application we employ the standard
scales~\cite{scaling}, as they are the most commonly used.

We recall their definitions, where $[n]\coloneqq\{1,\dots,n\}$ and
$\context_{1}\mid \context_{2}$ is the context apposition
operator. For crown scales we further require that $n\geq 3$.
\begin{align}\label{eq:standard-scales}
  \mathbb{N}_{n} &\coloneqq ([n],[n],=) \tag{Nominal Scale}\\
  \Ocon_{n} &\coloneqq ([n],[n],\leq)\tag{Ordinal Scale}\\
  \mathbb{I}_{n} &\coloneqq ([n],[n],\leq) \mid ([n],[n],\geq) \tag{Interordinal Scale}\\
  \mathbb{B}_{n} &\coloneqq ([n],[n],\neq)\tag{Contranominal Scale}\\
  \mathbb{C}_{n} &\coloneqq ([n],[n],J), \textit{ where } (a,b)\in J \iff a=b \text{ or } (a,b)=(n,1) \text{ or } b=a+1 \tag{Crown Scale}
\end{align}
We depict all these scales, more precisely their contextual
representations, in~\cref{fig:standard-scales1}. Additionally, we show
their corresponding concept lattices. Hence, our goal is to identify
these ordinal motifs (or ordinal patterns) in the spices data set. 
\input{pics/standard-scales.tex}

The number of identified local full scale-measure of the spices data
set per standard scale can be found in~\cref{tab:identify-motifs}. In
this table we distinguish between local and maximal local (with
respect to the heredity). We observe that the spices data set entails
a large number of ordinal motifs. The interordinal scale motifs are
the most frequent in both cases, i.e., local and maximal local. For
crown scales both values are equally 2145, since crown scales do not
have the heredity property. All found ordinal scale motifs are
trivial, i.e., all 37 found motifs are of size 1. In the last row
of~\cref{tab:identify-motifs} we printed the size of the largest
ordinal motif of the respective kind. Thus, the biggest motif is
nominal and of size nine. The largest crown is of size six. We
depicted all largest motifs in the
Appendix~\cref{fig:largest-nominal-fsm,fig:largest-interordinal-fsm,fig:largest-crown-fsm,fig:largest-contranominal-fsm}.

\begin{table}[b]
  \centering
  \caption{Results for ordinal motifs of the spices planner
    context. Every column represents ordinal motifs of a particular
    standard scale family.  Maximal lf-sm is the number of local full
    scale-measures for which there is no lf-sm with a larger domain.
    Largest lf-sm refers to the largest domain size that occurs in the
    set of local full scale-measures. }
  \begin{tabular}{|l|r|r|r|r|r|}
    \hline
    &nominal&ordinal&interordinal&contranominal&crown\\\hline \hline
    local full sm&2342&37&4643&2910&2145\\ \hline
    maximal lf-sm&527&37&2550&1498&2145\\ \hline
    largest lf-sm&9&1&5&5&6\\ \hline
  \end{tabular}
  \label{tab:identify-motifs}
\end{table}

\subsubsection*{Basic Meanings}
The discovered ordinal motifs allow us to interpret parts of the
spices data set in terms of their \emph{basic meaning} of standard
scales~\cite{fca-book}. In the following we provide basic meanings of the largest
local full scale-measure with respect to the found motifs. 

\begin{description}
\item[Nominal:] The food categories \emph{miscellaneous (group), fish
    (group), potato (group), vegetables (group), meat (group), sauce
    (group), poultry (group), rice (group)} and \emph{pastries
    (group)} \textbf{form a partition}.
\item[Ordinal:] There are no non trivial local full ordinal
  scale-measures. If this motif would exist in the spices data set, it
  would \textbf{form a rank order}.
\item[Interordinal:] The spices and food categories \emph{ginger,
    mugwort, meat (group), black pepper} and \emph{juniper berries}
  \textbf{form a linear betweenness relation}.
\item[Contranominal:] The spices \emph{Thyme, Sweet Paprika, Oregano,
    Caraway} and \emph{Black Pepper} \textbf{form a partition and
    are independent}.
\item[Crown:] The literature, precisely~\textcite{fca-book}, does not
  provide a basic meaning for crowns.
\end{description}

The ordinal motifs obviously allow a far more complex and meaningful
explanation of the substructures found. To develop this is the task of
future investigations.

\section{Discussion and Conclusion}
With our work we have shown a new approach to the analysis and
interpretation of ordinal data. By using scale-measures, we have found
an expressive representation for ordinal motifs that also allows us to
calculate and measure them. The necessary and useful extension of the
notion of scale-measures to include a local variant is a result that
will find applications in Formal Concept Analysis and beyond,
independent of ordinal motifs.

While our approach is capable to extract preset frequent recurring
ordinal patterns in order structures, there is room for
improvement. First, apart from our theoretical considerations of
computational complexity, we did not address the development of
specific algorithms. On the one hand, it is certainly possible to find
better algorithms in general than the naive implementations we used in
our experiments. On the other hand, there are special classes of
interesting ordinal motifs, such as the standard scales, which
certainly allow easier computations or even simpler computation
classes.

Second, in our example application, we have resorted to a very simple
interpretation of the ordinal motifs found. Here we can imagine that
with the help of researchers from the field of human computer
interaction, general as well as area-specific explanatory methods can
be derived. A third line of research would be an extension of the
notion of ordinal motifs towards other context-based pattern
languages, such as \emph{clones} \cite{clones}, \emph{p-clones}
\cite{pclones} or \emph{complements}~\cite{negative-attr}. Fourth, the
new ability to identify standard scales may help a common conceptual
data reduction method which is based on nested representations of
concept lattices \cite{FCA-visual}. Last, among the identified ordinal
motifs are artifacts of the underlying conceptual
scaling~\cite{scaling}. Those include a lot of trivial scales such as
$\emph{small}<\emph{medium}<\emph{large}$ which one may want to
remove.

We disclosed many lines of research on how to extend and improve our
methods. Improvements can be made both algorithmically for optimized
identification of specific ordinal motifs and in terms of the textual
explanations, by providing more understandable or domain specific
textual templates. Finally, studying the occurrence of ordinal motifs
quantitatively on a large number of data sets is the next task at hand.


\printbibliography

\clearpage

\appendix
\section{Appendix}

\begin{figure}
  \centering
  \input{pics/largest-nominal-fsm-spices.tikz}
  \caption{The largest local full nominal scale-measure of the spices
    data sets. We employed the dual context to get conceptual
    explanations of the attributes (spices). The attributes that
    induce the local full scale-measure are highlighted with bold
    font.  The diagram was rotated by 90 degrees counter clockwise to
    improve readability, i.e., the top concept is on the left.}
  \label{fig:largest-nominal-fsm}
\end{figure}

\begin{figure}
  \centering
  \input{pics/largest-interordinal-fsm-spices.tikz}
  \caption{The largest local full interordinal scale-measure of the
    spices data set.  We employed the dual context to get conceptual
    explanations of the attributes (spices). The attributes that
    induce the local full scale-measure are highlighted with bold
    font.}
  \label{fig:largest-interordinal-fsm}
\end{figure}

\begin{figure}
  \centering
  \input{pics/largest-contranominal-fsm-spices.tikz}
  \caption{The largest local full contranominal scale-measure of the
    spices data set. We employed the dual context to get conceptual
    explanations of the attributes (spices). The attributes that
    induce the local full scale-measure are highlighted with bold
    font.}
  \label{fig:largest-contranominal-fsm}
\end{figure}

\begin{figure}
  \centering
  \input{pics/largest-crown-fsm-spices.tikz}
  \caption{The largest local full crown scale-measure of the spices
    data set. We employed the dual context to get conceptual
    explanations of the attributes (spices). The attributes that
    induce the local full scale-measure are highlighted with bold
    font. The diagram was rotated by 90 degrees counter clockwise to
    improve readability, i.e., the top concept is on the left.}
  \label{fig:largest-crown-fsm}
\end{figure}

\end{document}

%% file: pics/standard-scales.tex
\begin{figure}
  \newcommand{\scale}{.75}
  \begin{minipage}{0.24\linewidth}
\scalebox{\scale}{\begin{cxt} 
  \cxtName{$\mathbb{N}_{n}$}
  \att{\rotatebox{90}{1}}
  \att{\rotatebox{90}{2}}
  \att{\rotatebox{90}{3}}
  \att{$\dots$}
  \att{\rotatebox{90}{$n-1$}}  
  \att{\rotatebox{90}{$n$}}  
  \obj{x.....}{1}
  \obj{.x....}{2}
  \obj{..x...}{3}
  \obj{......}{$\dots$}
  \obj{....x.}{$n-1$}
  \obj{.....x}{$n$}
\end{cxt}}    
  \end{minipage}
  \begin{minipage}{0.24\linewidth}\centering
\scalebox{\scale}{\begin{tikzpicture}[xscale=0.8,
    concept/.style={fill=black!36,shape=circle,draw=black!80},
  hidden/.style={opacity=0.5},
  edge/.style={color=gray},]
  \node[concept,hidden] (a) at (0,0){};
  \node[concept] (b) at (-2.5,1){};
  \node[concept] (c) at (-1.5,1){};
  \node[concept] (d) at (-0.5,1){};
  \node[concept,draw=white,fill=white] (e) at (0.5,1){$\dots$};
  \node[concept] (f) at (1.5,1){};
  \node[concept] (g) at (2.5,1){};
  \node[concept,hidden] (h) at (0,2){};
  \path[edge,hidden] (a) edge (b);
  \path[edge,hidden] (a) edge (c);
  \path[edge,hidden] (a) edge (d);

  \path[edge,hidden] (a) edge (f);
  \path[edge,hidden] (a) edge (g);
  \path[edge,hidden] (h) edge (b);
  \path[edge,hidden] (h) edge (c);
  \path[edge,hidden] (h) edge (d);

  \path[edge,hidden] (h) edge (f);
  \path[edge,hidden] (h) edge (g);
  \coordinate[label={above right:{1}}]() at (b);
  \coordinate[label={above right:{2}}]() at (c);
  \coordinate[label={above right:{3}}]() at (d);
  \coordinate[label={above left:{$n-1$}}]() at (f);
  \coordinate[label={above left:{$n$}}]() at (g);
  \coordinate[label={below right:{1}}]() at (b);
  \coordinate[label={below right:{2}}]() at (c);
  \coordinate[label={below right:{3}}]() at (d);
  \coordinate[label={below left:{$n-1$}}]() at (f);
  \coordinate[label={below left:{$n$}}]() at (g);
\end{tikzpicture}}
  \end{minipage}
  \begin{minipage}{0.24\linewidth}
\scalebox{\scale}{\begin{cxt} 
    \cxtName{$\mathbb{C}_{n}$}
    \att{\rotatebox{90}{1}}
    \att{\rotatebox{90}{2}}
    \att{\rotatebox{90}{3}}
    \att{$\dots$}
    \att{\rotatebox{90}{$n-1$}}  
    \att{\rotatebox{90}{$n$}}  
    \obj{xx....}{1}
    \obj{.xx...}{2}
    \obj{..xx..}{3}
    \obj{......}{$\dots$}
    \obj{....xx}{$n-1$}
    \obj{x....x}{$n$}
  \end{cxt}}
  \end{minipage}
  \begin{minipage}{0.24\linewidth}\centering
\scalebox{\scale}{ \begin{tikzpicture}[xscale=0.7,
    concept/.style={fill=black!36,shape=circle,draw=black!80},
                      hidden/.style={opacity=0.5},
                      edge/.style={color=gray}]
    \node[concept,hidden] (a) at (0,0){};
    \node[concept] (ba) at (-2.5,1){};
    \node[concept] (bb) at (-1.5,1){};
    \node[concept] (bc) at (-0.5,1){};
    \node[concept,draw=white,fill=white] (bd) at (0.5,1){$\dots$};
    \node[concept] (be) at (1.5,1){};
    \node[concept] (bf) at (2.5,1){};

    \coordinate[label={below right:{1}}](c) at (ba);
    \coordinate[label={below right:{2}}](c) at (bb);
    \coordinate[label={below right:{3}}](c) at (bc);

    \coordinate[label={below left:{$n-1$}}](c) at (be);
    \coordinate[label={below left:{$n$}}](c) at (bf);

    \node[concept] (ca) at (-2.5,2){};
    \node[concept] (cb) at (-1.5,2){};
    \node[concept] (cc) at (-0.5,2){};
    \node[concept,draw=white,fill=white] (cd) at (0.5,2){$\dots$};
    \node[concept] (ce) at (1.5,2){};
    \node[concept] (cf) at (2.5,2){};

    \coordinate[label={above right:{1}}](c) at (ca);
    \coordinate[label={above right:{2}}](c) at (cb);
    \coordinate[label={above right:{3}}](c) at (cc);

    \coordinate[label={above left:{$n-1$}}](c) at (ce);
    \coordinate[label={above left:{$n$}}](c) at (cf);

    \node[concept,hidden] (h) at (0,3){};
    \path[edge,hidden] (a) edge (ba);
    \path[edge,hidden] (a) edge (bb);
    \path[edge,hidden] (a) edge (bc);

    \path[edge,hidden] (a) edge (be);
    \path[edge,hidden] (a) edge (bf);

    \path[edge] (ba) edge (ca);
    \path[edge] (bb) edge (cb);
    \path[edge] (bc) edge (cc);

    \path[edge] (be) edge (ce);
    \path[edge] (bf) edge (cf);

    \path[edge] (ba) edge (cb);
    \path[edge] (bb) edge (cc);
    \path[edge,dashed] (bc) edge (cd);
    \path[edge,dashed] (bd) edge (ce);
    \path[edge] (be) edge (cf);

    \path[edge,in=-30,out=150] (bf) edge (ca);

    \path[edge,hidden] (h) edge (ca);
    \path[edge,hidden] (h) edge (cb);
    \path[edge,hidden] (h) edge (cc);

    \path[edge,hidden] (h) edge (ce);
    \path[edge,hidden] (h) edge (cf);
  \end{tikzpicture} }  
  \end{minipage}

  \begin{minipage}{0.24\linewidth}
\scalebox{\scale}{\begin{cxt} 
  \cxtName{$\Ocon_{n}$}
  \att{\rotatebox{90}{1}}
  \att{\rotatebox{90}{2}}
  \att{\rotatebox{90}{3}}
  \att{$\dots$}
  \att{\rotatebox{90}{$n-1$}}  
  \att{\rotatebox{90}{$n$}}  
  \obj{x.....}{1}
  \obj{xx....}{2}
  \obj{xxx...}{3}
  \obj{xxxx..}{$\dots$}
  \obj{xxxxx.}{$n-1$}
  \obj{xxxxxx}{$n$}
\end{cxt}}
  \end{minipage}
  \begin{minipage}{0.24\linewidth}\centering
\scalebox{\scale}{\begin{tikzpicture}[yscale=0.8, concept/.style={fill=black!36,shape=circle,draw=black!80},
  hidden/.style={opacity=0.5},
  edge/.style={color=gray}]
  \node[concept] (a) at (0,0){};
  \node[concept] (b) at (0,1){};
  \node[concept,draw=white,fill=white] (d) at (0,2){$\dots$};
  \node[concept] (e) at (0,3){};
  \node[concept] (f) at (0,4){};
  \path[edge] (a) edge (b);
  \path[edge] (b) edge (d);
  \path[edge] (d) edge (e);
  \path[edge] (e) edge (f);
  \coordinate[label={below right:{1}}]() at (f);
  \coordinate[label={below right:{2}}]() at (e);
  \coordinate[label={below right:{$n-1$}}]() at (b);
  \coordinate[label={below right:{$n$}}]() at (a);
  \coordinate[label={above right:{$n$}}]() at (a);
  \coordinate[label={above right:{$n-1$}}]() at (b);
  \coordinate[label={above right:{2}}]() at (e);
  \coordinate[label={above right:{1}}]() at (f);
\end{tikzpicture}}
  \end{minipage}
  \begin{minipage}{0.24\linewidth}
    \scalebox{\scale}{
\begin{cxt} 
  \cxtName{$\mathbb{B}_{n}$}
  \att{\rotatebox{90}{1}}
  \att{\rotatebox{90}{2}}
  \att{\rotatebox{90}{3}}
  \att{$\dots$}
  \att{\rotatebox{90}{$n-1$}}  
  \att{\rotatebox{90}{$n$}}    
  \obj{.xxxxx}{1}
  \obj{x.xxxx}{2}
  \obj{xx.xxx}{3}
  \obj{xxx.xx}{$\dots$}
  \obj{xxxx.x}{$n-1$}
  \obj{xxxxx.}{$n$}
\end{cxt}
    }
  \end{minipage}
  \begin{minipage}{0.24\linewidth}\centering
    \scalebox{\scale}{
  \begin{tikzpicture}[concept/.style={fill=black!36,shape=circle,draw=black!80},
                      hidden/.style={opacity=0.5},
                      edge/.style={color=gray},]
    \node[concept,hidden] (a) at (0,0){};
    \node[concept] (b) at (-1,1){};
    \node[concept] (c) at (0,1){};
    \node[concept] (d) at (1,1){};
    \node[concept] (e) at (-1,2){};
    \node[concept] (f) at (0,2){};
    \node[concept] (g) at (1,2){};
    \node[concept,hidden] (h) at (0,3){};
    \path[edge,hidden] (a) edge (b);
    \path[edge,hidden] (a) edge (c);
    \path[edge,hidden] (a) edge (d);
    \path[edge] (b) edge (e);
    \path[edge] (b) edge (f);
    \path[edge] (c) edge (e);
    \path[edge] (c) edge (g);
    \path[edge] (d) edge (f);
    \path[edge] (d) edge (g);
    \path[edge,hidden] (e) edge (h);
    \path[edge,hidden] (f) edge (h);
    \path[edge,hidden] (g) edge (h);

    \coordinate[label={below right:{1}}]() at (b);
    \coordinate[label={below right:{2}}]() at (c);
    \coordinate[label={below right:{3}}]() at (d);

    \coordinate[label={above right:{3}}]() at (e);
    \coordinate[label={above right:{2}}]() at (f);
    \coordinate[label={above right:{1}}]() at (g);

  \end{tikzpicture}
    }
  \end{minipage}

  \begin{minipage}{0.49\linewidth}\centering
    \scalebox{\scale}{
      \begin{cxt} 
        \cxtName{$\mathbb{I}_{n}$}
        \att{\rotatebox{90}{$1_\leq$}}
        \att{\rotatebox{90}{$2_\leq$}}
        \att{\rotatebox{90}{$3_\leq$}}
        \att{$\dots$}
        \att{\rotatebox{90}{$(n-1)_\leq$}}  
        \att{\rotatebox{90}{$n_\leq$}}  
        \att{\rotatebox{90}{$1_\geq$}}
        \att{\rotatebox{90}{$2_\geq$}}
        \att{\rotatebox{90}{$3_\geq$}}
        \att{$\dots$}
        \att{\rotatebox{90}{$(n-1)_\geq$}}  
        \att{\rotatebox{90}{$n_\geq$}}  
        \obj{xxxxxxx.....}{1}
        \obj{.xxxxxxx....}{2}
        \obj{..xxxxxxx...}{3}
        \obj{...xxxxxxx..}{$\dots$}
        \obj{....xxxxxxx.}{$n-1$}
        \obj{.....xxxxxxx}{$n$}
      \end{cxt}
    }
  \end{minipage}
  \begin{minipage}{0.49\linewidth}\centering
    \scalebox{\scale}{
\begin{tikzpicture}[concept/.style={fill=black!36,shape=circle,draw=black!80},
  hidden/.style={opacity=0.5},
  edge/.style={color=gray},]

  \node[concept] (d) at (0,5){};

  \node[concept] (c) at (-0.5,4){};
  \node[concept] (ad) at (0.5,4){};

  \node[concept] (b) at (-1,3){};
  \node[concept] (m1) at (0,3){};
  \node[concept] (bd) at (1,3){};

  \node[concept] (a) at (-1.5,2){};
  \node[concept] (m2) at (-0.5,2){};
  \node[concept] (m3) at (0.5,2){};
  \node[concept] (cd) at (1.5,2){};

  \node[concept,hidden] (bot) at (0,1){};

  \path[edge,hidden] (a) edge (bot);
  \path[edge,hidden] (m2) edge (bot);
  \path[edge,hidden] (m3) edge (bot);
  \path[edge,hidden] (cd) edge (bot);

  \path[edge,thick] (b) edge (a);
  \path[edge,thick] (b) edge (m2);
  \path[edge,thick] (m1) edge (m2);
  \path[edge,thick] (m1) edge (m3);
  \path[edge,thick] (bd) edge (m3); 
  \path[edge,thick] (bd) edge (cd); 

  \path[edge,thick] (c) edge (b);
  \path[edge,thick] (c) edge (m1);
  \path[edge,thick] (ad) edge (m1);
  \path[edge,thick] (ad) edge (bd);

  \path[edge,thick] (d) edge (c);
  \path[edge,thick] (d) edge (ad);

  \coordinate[label={above:{$4_\leq,1_\geq,$}}]() at (d);
  \coordinate[label={above left:{$3_\leq$}}]() at (c);
  \coordinate[label={above left:{$2_\leq$}}]() at (b);
  \coordinate[label={above left:{$1_\leq$}}]() at (a);

  \coordinate[label={above right:{$2_\geq$}}]() at (ad);
  \coordinate[label={above right:{$3_\geq$}}]() at (bd);
  \coordinate[label={above right:{$4_\geq$}}]() at (cd);

  \coordinate[label={below right:{1}}]() at (a);
  \coordinate[label={below right:{2}}]() at (m2);
  \coordinate[label={below right:{3}}]() at (m3);
  \coordinate[label={below right:{4}}]() at (cd);
\end{tikzpicture}
    }    
  \end{minipage}

  \caption{In this figure we depict the formal contexts and concept
    lattices of standard scales. From top left to bottom right we
    depicted the \emph{nominal scale} $\mathbb{N}_{n}$, the
    \emph{crown scale} $\mathbb{C}_{n}$, the \emph{ordinal scale}
    $\mathbb{O}_{n}$, the \emph{contranominal scale} $\mathbb{B}_{n}$
    and the \emph{interordinal scale} $\mathbb{I}_{n}$.}
  \label{fig:standard-scales1}
\end{figure}


%% file: pics/largest-nominal-fsm-spices.tikz
\colorlet{mivertexcolor}{black!80}
\colorlet{jivertexcolor}{black!80}
\colorlet{vertexcolor}{black!80}
\colorlet{bordercolor}{black!80}
\colorlet{linecolor}{gray}
\tikzset{vertexbase/.style 2 args={semithick, shape=circle, inner sep=2pt, outer sep=0pt, draw=bordercolor},%
  vertex/.style 2 args={vertexbase={#1}{}, fill=vertexcolor!45},%
  mivertex/.style 2 args={vertexbase={#1}{}, fill=mivertexcolor!45},%
  jivertex/.style 2 args={vertexbase={#1}{}, fill=jivertexcolor!45},%
  divertex/.style 2 args={vertexbase={#1}{}, top color=mivertexcolor!45, bottom color=jivertexcolor!45},%
  conn/.style={-, thick, color=linecolor}%
}
\tikzstyle{o} = [text width=8cm, align=center, font=\tiny]
\begin{tikzpicture}[xscale=1.4,yscale=1.5,rotate=90]
  \begin{scope} 
    \begin{scope} 
      \foreach \nodename/\nodetype/\param/\xpos/\ypos in {%
        0/vertex//5/1,
        1/divertex//1/2,
        2/divertex//2/2,
        3/divertex//3/2,
        4/divertex//4/2,
        5/divertex//5/2,
        6/divertex//6/2,
        7/divertex//7/2,
        8/divertex//8/2,
        9/divertex//9/2,
        10/vertex//5/3
      } \node[\nodetype={\param}{}] (\nodename) at (\xpos, \ypos) {};
    \end{scope}
    \begin{scope} 
      \path (0) edge[conn] (3);
      \path (0) edge[conn] (6);
      \path (0) edge[conn] (1);
      \path (1) edge[conn] (10);
      \path (0) edge[conn] (8);
      \path (0) edge[conn] (9);
      \path (0) edge[conn] (2);
      \path (2) edge[conn] (10);
      \path (0) edge[conn] (5);
      \path (7) edge[conn] (10);
      \path (4) edge[conn] (10);
      \path (6) edge[conn] (10);
      \path (9) edge[conn] (10);
      \path (0) edge[conn] (4);
      \path (3) edge[conn] (10);
      \path (8) edge[conn] (10);
      \path (5) edge[conn] (10);
      \path (0) edge[conn] (7);
    \end{scope}
    \begin{scope} 
      \foreach \nodename/\labelpos/\labelopts/\labelcontent in {%
        1/above/o/{punch, desserts, cheese cookies, cakes, compote/jam, punch/tea, fruit salad, christmas cookies},
        2/above/o/{herb curd/dips, omelettes, pizza, mushrooms},
        3/above/o/{steamed fish, fried fish, grilled fish, crustaceans/shellfish, baked fish},
        4/above/o/{fried potatoes, potato casserole/gratin, mashed potatoes, potato soup, oven potatoes},
        5/above/o/{carrots, red cabbage, leaf lettuce, spinach, vegetable casserole/gratin, broccoli, cauliflower, pea/bean/lentil soup, vegetable soup/minestrone, cucumbers/salad, preserves, beans, sauerkraut, tomatoes/salad, kohlrabi},
        6/above/o/{risotto/paella, asian rice table, rice pudding, curry rice},
        7/above/o/{beef, goulash, minced meat, lamb, sauerbraten, rouladen, veal, pork, game},
        8/above/o/{duck, goose, chicken/turkey},
        9/above/o/{light sauces, tomato based pasta sauces, dark sauces},
        1/below/o/{\textbf{pastries (group)}},
        2/below/o/{\textbf{miscellaneous (group)}},
        3/below/o/{\textbf{fish (group)}},
        4/below/o/{\textbf{potato (group)}},
        5/below/o/{\textbf{vegetables (group)}},
        6/below/o/{\textbf{rice (group)}},
        7/below/o/{\textbf{meat (group)}},
        8/below/o/{\textbf{poultry (group)}},
        9/below/o/{\textbf{sauces (group)}}
      } \coordinate[label={[\labelopts]\labelpos:{\labelcontent}}](c) at (\nodename);
    \end{scope}
  \end{scope}
\end{tikzpicture}


%% file: pics/largest-interordinal-fsm-spices.tikz
\colorlet{mivertexcolor}{black!80}
\colorlet{jivertexcolor}{black!80}
\colorlet{vertexcolor}{black!80}
\colorlet{bordercolor}{black!80}
\colorlet{linecolor}{gray}
\tikzset{vertexbase/.style 2 args={semithick, shape=circle, inner sep=2pt, outer sep=0pt, draw=bordercolor},%
  vertex/.style 2 args={vertexbase={#1}{}, fill=vertexcolor!45},%
  mivertex/.style 2 args={vertexbase={#1}{}, fill=mivertexcolor!45},%
  jivertex/.style 2 args={vertexbase={#1}{}, fill=jivertexcolor!45},%
  divertex/.style 2 args={vertexbase={#1}{}, top color=mivertexcolor!45, bottom color=jivertexcolor!45},%
  conn/.style={-, thick, color=linecolor}%
}
\tikzstyle{o} = [text width=3cm, align=center, font=\tiny]
\tikzstyle{r} = [text width=3cm, align=right, font=\tiny]
\tikzstyle{l} = [text width=3cm, align=left, font=\tiny]
\tikzstyle{b} = [text width=8cm, align=center, font=\tiny]
\begin{tikzpicture}[xscale=0.35,yscale=0.4]
  \begin{scope} 
    \begin{scope} 
      \foreach \nodename/\nodetype/\param/\xpos/\ypos in {%
        0/vertex//0.3826086956521735/4.143478260869564,
        1/jivertex//0.12753623188405783/9.213043478260868,
        2/jivertex//-5.866666666666667/9.244927536231884,
        3/divertex//-12.402898550724641/9.404347826086955,
        4/jivertex//6.7594202898550755/9.436231884057971,
        5/divertex//12.944927536231884/9.818840579710145,
        6/mivertex//-9.310144927536236/12.847826086956523,
        7/vertex//-3.0927536231884076/12.943478260869565,
        8/vertex//3.3478260869565197/13.03913043478261,
        9/mivertex//10.202898550724637/13.102898550724639,
        10/mivertex//-6.249275362318844/16.355072463768117,
        11/vertex//-0.03188405797101623/16.482608695652175,
        12/mivertex//6.791304347826088/16.73768115942029,
        13/mivertex//-3.1884057971014492/19.894202898550724,
        14/mivertex//3.315942028985507/20.308695652173917,
        15/vertex//-0.06376811594202891/23.46521739130435
      } \node[\nodetype={\param}{}] (\nodename) at (\xpos, \ypos) {};
    \end{scope}
    \begin{scope} 
      \path (6) edge[conn] (10);
      \path (9) edge[conn] (12);
      \path (4) edge[conn] (9);
      \path (4) edge[conn] (8);
      \path (7) edge[conn] (10);
      \path (7) edge[conn] (11);
      \path (14) edge[conn] (15);
      \path (8) edge[conn] (12);
      \path (8) edge[conn] (11);
      \path (2) edge[conn] (6);
      \path (2) edge[conn] (7);
      \path (10) edge[conn] (13);
      \path (1) edge[conn] (7);
      \path (1) edge[conn] (8);
      \path (12) edge[conn] (14);
      \path (0) edge[conn] (4);
      \path (0) edge[conn] (2);
      \path (0) edge[conn] (1);
      \path (0) edge[conn] (5);
      \path (0) edge[conn] (3);
      \path (5) edge[conn] (9);
      \path (13) edge[conn] (15);
      \path (3) edge[conn] (6);
      \path (11) edge[conn] (14);
      \path (11) edge[conn] (13);
    \end{scope}
    \begin{scope} 
      \foreach \nodename/\labelpos/\labelopts/\labelcontent in {%
        0/above/b/{steamed fish, fried potatoes, fried fish, risotto/paella, leafy salad, omlette, potato casserole/gratin, vegetable casserole/gratin, tomato based pasta sauces, pizza, desserts, puree, cheese cookies, broccoli, grilled fish, potato soup, asian rice dish, baked fish, rice pudding, oven potatoes, cauliflower, vegetable soup/minestrone, fruit salad, cucumber/salad, beans, curried rice},
        1/above/o/{goulash, veal},
        2/above/o/{goose},
        3/above/o/{carrots, light sauces, spinach, punch, cake, compote/jam, punch/tea, christmas cookies, chicken/turkey.},
        4/above/l/{herb curd/dips, mushrooms, pea/bean/lentil soup, tomato/salad},
        5/above/r/{crustaceans/shellfish, sauerkraut},
        6/above/o/{duck},
        8/above/o/{beef, roulades},
        9/above/o/{red cabbage, preserves, dark sauces, kohlrabi},
        10/above/o/{minced meat, pork},
        12/above/o/{sauerbraten},
        13/above/o/{lamb meat},
        14/above/o/{game meat},
        1/below/o/{\textbf{meat (group)}},
        2/below/o/{\textbf{mugwort}},
        3/below/o/{\textbf{ginger}},
        4/below/o/{\textbf{black pepper}},
        5/below/o/{\textbf{juniper berries}}
      } \coordinate[label={[\labelopts]\labelpos:{\labelcontent}}](c) at (\nodename);
    \end{scope}
  \end{scope}
\end{tikzpicture}


%% file: pics/largest-contranominal-fsm-spices.tikz
\colorlet{mivertexcolor}{black!80}
\colorlet{jivertexcolor}{black!80}
\colorlet{vertexcolor}{black!80}
\colorlet{bordercolor}{black!80}
\colorlet{linecolor}{gray}
\tikzset{vertexbase/.style 2 args={semithick, shape=circle, inner sep=2pt, outer sep=0pt, draw=bordercolor},%
  vertex/.style 2 args={vertexbase={#1}{}, fill=vertexcolor!45},%
  mivertex/.style 2 args={vertexbase={#1}{}, fill=mivertexcolor!45},%
  jivertex/.style 2 args={vertexbase={#1}{}, fill=jivertexcolor!45},%
  divertex/.style 2 args={vertexbase={#1}{}, top color=mivertexcolor!45, bottom color=jivertexcolor!45},%
  conn/.style={-, thick, color=linecolor}%
}
\tikzstyle{c} = [text width=3cm, align=center, font=\tiny]
\tikzstyle{c2} = [text width=1.5cm, align=center, font=\tiny]
\tikzstyle{r} = [text width=3cm, align=right, font=\tiny]
\tikzstyle{r2} = [text width=1.5cm, align=right, font=\tiny]
\tikzstyle{l} = [text width=3cm, align=left, font=\tiny]
\tikzstyle{l2} = [text width=1.5cm, align=left, font=\tiny]
\tikzstyle{b} = [text width=10cm, align=center, font=\tiny]
\begin{tikzpicture}[scale=0.11]
  \begin{scope} 
    \begin{scope} 
      \foreach \nodename/\nodetype/\param/\xpos/\ypos in {%
        0/vertex//50.989966626992484/3.6470629419589358,
        1/jivertex//50.583082919088014/16.396085789633318,
        2/jivertex//62.11145464304887/16.667341594902993,
        3/jivertex//39.19033909776194/16.802969497537816,
        4/vertex//62.65396625358822/30.501387663656047,
        5/vertex//50.447455016453176/31.043899274195383,
        6/vertex//39.054711195127105/31.315155079465057,
        7/jivertex//90.32205839109437/32.2645503979089,
        8/jivertex//7.819835235361435/36.09193120026133,
        9/vertex//50.447455016453176/42.843526803425945,
        10/vertex//102.25731382295976/43.79292212186977,
        11/vertex//90.59331419636402/44.199805829774284,
        12/vertex//78.5224308618638/44.471061635043945,
        13/vertex//19.020607248710682/46.5413829961025,
        14/vertex//7.627863427384632/46.948266704007,
        15/vertex//-3.9005082965762554/47.21952250927667,
        16/vertex//102.79982543349908/55.99943335900483,
        17/vertex//91.07151345422297/56.62994328922494,
        18/vertex//79.06494247240315/57.35571238535317,
        19/vertex//50.53140096618357/59.29951690821256,
        20/vertex//-3.9005082965762554/59.83291745431623,
        21/vertex//19.15623515134552/59.83291745431623,
        22/vertex//7.394714462414573/60.748936031179205,
        23/mivertex//90.8645700016337/69.15533991458372,
        24/vertex//38.946859903381636/70.24637681159422,
        25/vertex//50.21256038647343/70.8840579710145,
        26/vertex//61.58454106280193/71.09661835748793,
        27/mivertex//7.182154075941142/72.54603748045457,
        28/mivertex//38.888888888888886/80.0,
        29/mivertex//50.0/80.0,
        30/mivertex//61.111111111111114/80.0,
        31/vertex//49.49805969800934/91.53394384933134
      } \node[\nodetype={\param}{}] (\nodename) at (\xpos, \ypos) {};
    \end{scope}
    \begin{scope} 
      \path (18) edge[conn] (23);
      \path (9) edge[conn] (27);
      \path (2) edge[conn] (4);
      \path (10) edge[conn] (17);
      \path (7) edge[conn] (10);
      \path (25) edge[conn] (30);
      \path (13) edge[conn] (21);
      \path (19) edge[conn] (26);
      \path (7) edge[conn] (19);
      \path (7) edge[conn] (11);
      \path (29) edge[conn] (31);
      \path (2) edge[conn] (10);
      \path (15) edge[conn] (24);
      \path (0) edge[conn] (2);
      \path (6) edge[conn] (9);
      \path (9) edge[conn] (23);
      \path (8) edge[conn] (19);
      \path (25) edge[conn] (28);
      \path (26) edge[conn] (30);
      \path (17) edge[conn] (29);
      \path (2) edge[conn] (5);
      \path (17) edge[conn] (23);
      \path (14) edge[conn] (25);
      \path (0) edge[conn] (8);
      \path (11) edge[conn] (18);
      \path (3) edge[conn] (15);
      \path (12) edge[conn] (18);
      \path (13) edge[conn] (22);
      \path (5) edge[conn] (22);
      \path (20) edge[conn] (28);
      \path (4) edge[conn] (21);
      \path (3) edge[conn] (6);
      \path (1) edge[conn] (14);
      \path (24) edge[conn] (28);
      \path (23) edge[conn] (31);
      \path (18) edge[conn] (28);
      \path (26) edge[conn] (29);
      \path (5) edge[conn] (17);
      \path (13) edge[conn] (26);
      \path (4) edge[conn] (16);
      \path (27) edge[conn] (31);
      \path (4) edge[conn] (9);
      \path (19) edge[conn] (24);
      \path (1) edge[conn] (6);
      \path (10) edge[conn] (26);
      \path (12) edge[conn] (17);
      \path (21) edge[conn] (30);
      \path (3) edge[conn] (12);
      \path (20) edge[conn] (27);
      \path (15) edge[conn] (22);
      \path (22) edge[conn] (29);
      \path (30) edge[conn] (31);
      \path (21) edge[conn] (27);
      \path (11) edge[conn] (16);
      \path (1) edge[conn] (4);
      \path (8) edge[conn] (14);
      \path (10) edge[conn] (16);
      \path (24) edge[conn] (29);
      \path (1) edge[conn] (11);
      \path (0) edge[conn] (3);
      \path (8) edge[conn] (13);
      \path (0) edge[conn] (7);
      \path (6) edge[conn] (18);
      \path (19) edge[conn] (25);
      \path (16) edge[conn] (23);
      \path (3) edge[conn] (5);
      \path (11) edge[conn] (25);
      \path (0) edge[conn] (1);
      \path (8) edge[conn] (15);
      \path (5) edge[conn] (9);
      \path (15) edge[conn] (20);
      \path (7) edge[conn] (12);
      \path (14) edge[conn] (21);
      \path (12) edge[conn] (24);
      \path (28) edge[conn] (31);
      \path (22) edge[conn] (27);
      \path (16) edge[conn] (30);
      \path (2) edge[conn] (13);
      \path (14) edge[conn] (20);
      \path (6) edge[conn] (20);
    \end{scope}
    \begin{scope} 
      \foreach \nodename/\labelpos/\labelopts/\labelcontent in {%
        0/above/b/{carrots, leaf salad, spinach, punch, desserts, puree, cake, potato soup, asian rice table, crustaceans/shellfish, baked fish, rice pudding, compote/jam, cauliflower, punch/tea, fruit salad, christmas cookies, cucumbers/salad, goose, beans, curried rice.},
        1/above/c2/{red cabbage, sauerbraten},
        2/above/r2/{broccoli, sauerkraut},
        4/above/c/{pea/bean/lentil soup, kohlrabi},
        6/above/c/{tomato/-salat},
        7/above/c/{steamed fish, omlette, light sauces, grilled fish, veal},
        8/above/c/{risotto/paella, minced meat, chicken/turkey},
        10/above/r2/{potato casserole / gratin},
        11/above/c/{game meat},
        12/above/c/{fried fish},
        13/above right/l2/{goulash, cheese pastry},
        14/above/c/{roulades},
        15/above/l2/{vegetable casserole gratin},
        19/above/c/{fried potatoes, duck, pork},
        23/above/c/{lamb leat},
        24/above/l2/{tomato based pasta sauces, pizza},
        25/above/c2/{beef, mushrooms},
        26/above right/l/{oven potato},
        27/above/c/{herb curd/dips},
        28/above/l2/{stew},
        29/above/c2/{vegetable soup},
        30/above/r2/{dark sauce},
        1/below/c2/{\textbf{black pepper}},
        2/below right/l/{\textbf{caraway}},
        3/below left/r/{\textbf{oregano}},
        7/below/c/{\textbf{thyme}},
        8/below/c2/{\textbf{sweet paprika}}
      } \coordinate[label={[\labelopts]\labelpos:{\labelcontent}}](c) at (\nodename);
    \end{scope}
  \end{scope}
\end{tikzpicture}


%% file: pics/largest-crown-fsm-spices.tikz
\colorlet{mivertexcolor}{black!80}
\colorlet{jivertexcolor}{black!80}
\colorlet{vertexcolor}{black!80}
\colorlet{bordercolor}{black!80}
\colorlet{linecolor}{gray}
\tikzset{vertexbase/.style 2 args={semithick, shape=circle, inner sep=2pt, outer sep=0pt, draw=bordercolor},%
  vertex/.style 2 args={vertexbase={#1}{}, fill=vertexcolor!45},%
  mivertex/.style 2 args={vertexbase={#1}{}, fill=mivertexcolor!45},%
  jivertex/.style 2 args={vertexbase={#1}{}, fill=jivertexcolor!45},%
  divertex/.style 2 args={vertexbase={#1}{}, top color=mivertexcolor!45, bottom color=jivertexcolor!45},%
  conn/.style={-, thick, color=linecolor}%
}
\tikzstyle{c} = [text width=3cm, align=center, font=\tiny]
\tikzstyle{c2} = [text width=1.5cm, align=center, font=\tiny]
\tikzstyle{r} = [text width=3cm, align=right, font=\tiny]
\tikzstyle{r2} = [text width=1.5cm, align=right, font=\tiny]
\tikzstyle{l} = [text width=4cm, align=left, font=\tiny]
\tikzstyle{l2} = [text width=1.5cm, align=left, font=\tiny]
\tikzstyle{b} = [text width=8cm, align=right, font=\tiny]
\begin{tikzpicture}[xscale=2.5,yscale=2,rotate=90]
  \begin{scope} 
    \begin{scope} 
      \foreach \nodename/\nodetype/\param/\xpos/\ypos in {%
        0/vertex//2.3/1,
        1/jivertex//1.8/2,
        2/jivertex//2.8/2,
        3/jivertex//0/2,
        4/jivertex//3.8/2,
        5/jivertex//4.9/2,
        6/jivertex//0.86/2,
        7/mivertex//4.9/3,
        8/mivertex//3.8/3,
        9/mivertex//2.8/3,
        10/mivertex//0/3,
        11/mivertex//1.8/3,
        12/mivertex//0.86/3,
        13/vertex//2.3/4
      } \node[\nodetype={\param}{}] (\nodename) at (\xpos, \ypos) {};
    \end{scope}
    \begin{scope} 
      \path (0) edge[conn] (1);
      \path (1) edge[conn] (9);
      \path (5) edge[conn] (10);
      \path (3) edge[conn] (10);
      \path (0) edge[conn] (5);
      \path (0) edge[conn] (3);
      \path (4) edge[conn] (7);
      \path (2) edge[conn] (9);
      \path (7) edge[conn] (13);
      \path (5) edge[conn] (7);
      \path (1) edge[conn] (11);
      \path (2) edge[conn] (8);
      \path (12) edge[conn] (13);
      \path (8) edge[conn] (13);
      \path (0) edge[conn] (4);
      \path (9) edge[conn] (13);
      \path (6) edge[conn] (12);
      \path (3) edge[conn] (12);
      \path (6) edge[conn] (11);
      \path (4) edge[conn] (8);
      \path (0) edge[conn] (6);
      \path (11) edge[conn] (13);
      \path (10) edge[conn] (13);
      \path (0) edge[conn] (2);
    \end{scope}
    \begin{scope} 
      \foreach \nodename/\labelpos/\labelopts/\labelcontent in {%
        0/above/b/{beef, roast potatoes, carrots, omlettes, spinach, mashed potatoes, mushrooms, potato soup, roulades, cauliflower, cucumbers/salad, beans, curry rice, chicken/turkey},
        1/above right/l/{crustaceans/shellfish, baked fish},
        2/above right/l/{asian rice table, rice pudding, fruit salad},
        3/above right/l/{goulash, cheese cookies, broccoli, baked potato, pea/bean/lentil soup, vegetable soup/minestrone, dark sauces, sauerkraut, kohlrabi},
        4/above right/l/{red cabbage, sauerbraten, veal},
        5/above right/l/{minced meat, goose},
        6/above right/l/{risotto/paella, leaf salad, light sauces, vegetable casserole/gratin, tomato-based pasta sauces, pizza, potstickers, tomato/salad},
        7/above/c/{duck, pork, game},
        8/above/c/{punch, desserts, cakes, compote/jam, punch/tea, christmas cookies},
        9/above/c/{steamed fish},
        10/above/c/{lamb meat},
        11/above/c/{fried fish, grilled fish},
        12/above/c/{herb curd/dips, potato casserole/gratin},
        1/below/c/{\textbf{fish (group)}},
        2/below/c/{\textbf{anis}},
        3/below/c/{\textbf{caraway}},
        4/below/c/{\textbf{cloves}},
        5/below/c/{\textbf{mugwort}},
        6/below/c/{\textbf{basil}}
      } \coordinate[label={[\labelopts]\labelpos:{\labelcontent}}](c) at (\nodename);
    \end{scope}
  \end{scope}
\end{tikzpicture}
